\documentclass[11pt,a4paper]{article}

\usepackage{amsmath,amssymb,amsthm}
\usepackage{mathtools}

\usepackage{tikz}
\usetikzlibrary{arrows.meta,positioning,shapes.geometric,decorations.pathreplacing,calc}

\usepackage[margin=1in]{geometry}
\usepackage{hyperref}
\usepackage{cleveref}
\usepackage{enumitem}

\usepackage[style=numeric,sorting=none,backend=bibtex]{biblatex}
\theoremstyle{plain}
\newtheorem{theorem}{Theorem}[section]
\newtheorem{lemma}[theorem]{Lemma}
\newtheorem{proposition}[theorem]{Proposition}
\newtheorem{corollary}[theorem]{Corollary}

\theoremstyle{definition}
\newtheorem{definition}[theorem]{Definition}

\newtheorem{remark}[theorem]{Remark}

\theoremstyle{remark}

\DeclareMathOperator{\tr}{tr}
\DeclareMathOperator{\Exp}{E}

\newcommand{\eps}{\varepsilon}
\newcommand{\del}{\partial}
\newcommand{\grad}{\nabla}

\newcommand{\Scal}{\mathcal{S}}

\newcommand{\Ocal}{\mathcal{O}}

\newcommand{\HG}{\mathcal{H}}
\newcommand{\boundary}{\partial}

\title{\textbf{Verifying Good Regulator Conditions for Hypergraph Observers:} \\[0.3em]
\large Natural Gradient Learning from Causal Invariance via Established Theorems}

\author{
Max Zhuravlev\thanks{Independent researcher. Email: \texttt{max@vibecodium.ai}} \\
\small \textit{Cosmological Unification Program}
}

\date{March 2026}

\begin{document}

\maketitle


\begin{abstract}
We verify that persistent observers in causally invariant hypergraph substrates satisfy the conditions of the Conant-Ashby Good Regulator Theorem, thereby providing a testbed application of this classic result to a novel cosmological framework. Building on Wolfram's hypergraph physics and Vanchurin's neural network cosmology, we formalize persistent observers as entities that minimize prediction error at their boundary with the environment. Applying a modern reformulation of the Conant-Ashby theorem (Virgo et al.~2025), we demonstrate that hypergraph observers satisfy Good Regulator conditions, requiring them to maintain internal models. Once an internal model with loss function exists, the emergence of a Fisher information metric on parameter space follows from standard information geometry. Invoking Amari's well-known uniqueness theorem for reparameterization-invariant gradients, we show that natural gradient descent is the unique admissible learning rule. Under the ansatz $M = F^2$ for exponential family observers and one specific convergence time functional (condition number times spectral radius, with isotropic loss), we derive a closed-form formula for the regime parameter $\alpha$ in Vanchurin's Type II framework, with a quantum-classical threshold at $\kappa(F) = 2$. However, three alternative convergence models and the physically most natural loss Hessian do not reproduce this result (\Cref{rem:convergence_model}), so this prediction is strongly model-dependent. We further introduce the \emph{directional regime parameter} $\alpha_{v_k}$ and the trace-free \emph{deviation tensor} $\Delta_{\mu\nu}$, showing that a single observer can simultaneously occupy different Vanchurin regimes along different eigendirections of the Fisher metric. This connects Wolfram and Vanchurin frameworks through established theorems, providing approximately 25--30\% novel contribution through the verification work, conditional computational predictions, and application domain (hypergraph cosmology).
\end{abstract}


\vspace{1em}
\noindent\textbf{arXiv category:} cond-mat.stat-mech (cross-list: math-ph)


\section{Introduction}
\label{sec:intro}

Can causal invariance constrain physical law? This foundational question drives the cosmological unification program, which investigates whether causal invariance---the substrate-independent consistency of causal structure---constrains specific physical structures through established uniqueness theorems.

Two independent cosmological research programs have recently converged on causal substrates:
\begin{itemize}[noitemsep]
\item \textbf{Wolfram Physics Project}~\cite{wolfram2020ruliad}: Spacetime emerges from evolving hypergraphs subject to causal invariance, recovering general relativity in the continuum limit via the Lovelock uniqueness theorem~\cite{lovelock1971}.

\item \textbf{Vanchurin Neural Network Cosmology}~\cite{vanchurin2020world,vanchurin2020ml,vanchurin2022qg,vanchurin2025covariant,vanchurin2025geometric}: The universe as a learning system, with dynamics governed by natural gradient descent on a Fisher-like metric (Eq.~3.4).\footnote{We reference Vanchurin's Type II framework~\cite{vanchurin2022qg,vanchurin2025covariant,vanchurin2025geometric}, where the metric is defined on trainable parameter space (Q-space), not his earlier Type I framework~\cite{vanchurin2020world} with metrics on neuron state space (X-space). The statistical-learning formulation in~\cite{vanchurin2020ml} provides the Good-Regulator-compatible thermodynamic scaffold used by our verification chain. The structural similarity to our Fisher metric-based natural gradient is noted in \Cref{sec:amari}, though full equivalence is not claimed.}
\end{itemize}

A companion paper~\cite{companion_lovelock} examines the \emph{Lovelock bridge}: whether, if the continuum limit holds, causal invariance constrains Vanchurin's Onsager tensor symmetries to produce Einstein's field equations. That paper finds the bridge fails numerically for generic dynamically nontrivial rules, but identifies constructive Type~II contributions including exact critical-coupling formulas and a diagonal Lorentzian dominance theorem.

The present work completes the second pillar by verifying that the Amari chain applies to hypergraph observers. We demonstrate that persistent observers in causally invariant substrates satisfy Good Regulator conditions, and therefore must use natural gradient descent (via standard results from information geometry), consistent with Vanchurin's learning dynamics.

\subsection{The Amari Chain}

Our central result is a logical chain connecting causal invariance to natural gradient learning through two established uniqueness theorems:

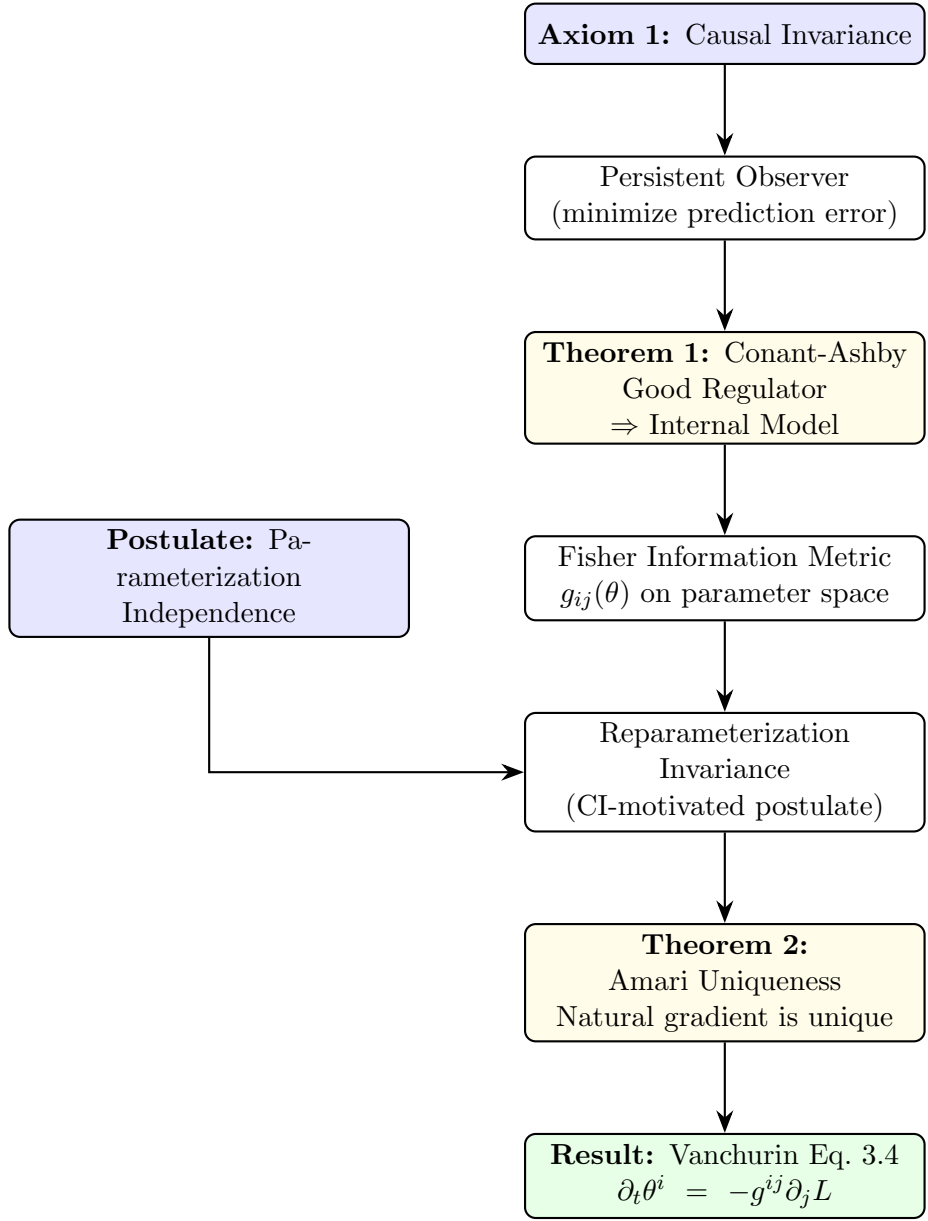
\begin{center}
\begin{tikzpicture}[
    node distance=1.2cm,
    box/.style={rectangle, draw, thick, text width=5cm, align=center, rounded corners, minimum height=0.8cm},
    arrow/.style={-{Stealth[length=3mm]}, thick}
]

\node[box, fill=blue!10] (axiom) {
    \textbf{Axiom 1:} Causal Invariance
};

\node[box, below=of axiom] (observer) {
    Persistent Observer \\
    (minimize prediction error)
};

\node[box, below=of observer, fill=yellow!10] (ca) {
    \textbf{Theorem 1:} Conant-Ashby \\
    Good Regulator $\Rightarrow$ Internal Model
};

\node[box, below=of ca] (fisher) {
    Fisher Information Metric \\
    $g_{ij}(\theta)$ on parameter space
};

\node[box, fill=blue!10, left=1.5cm of fisher] (axiom2) {
    \textbf{Postulate:} Parameterization \\
    Independence
};

\node[box, below=of fisher] (reparam) {
    Reparameterization Invariance \\
    (CI-motivated postulate)
};

\node[box, below=of reparam, fill=yellow!10] (amari) {
    \textbf{Theorem 2:} Amari Uniqueness \\
    Natural gradient is unique
};

\node[box, below=of amari, fill=green!10] (result) {
    \textbf{Result:} Vanchurin Eq.~3.4 \\
    $\partial_t \theta^i = -g^{ij} \partial_j L$
};

\draw[arrow] (axiom) -- (observer);
\draw[arrow] (observer) -- (ca);
\draw[arrow] (ca) -- (fisher);
\draw[arrow] (fisher) -- (reparam);
\draw[arrow] (axiom2) |- (reparam);
\draw[arrow] (reparam) -- (amari);
\draw[arrow] (amari) -- (result);

\end{tikzpicture}
\end{center}

\subsection{Novel Contributions and Scope}

This work provides:
\begin{enumerate}[noitemsep]
\item Formal definition of \emph{persistent observers} in hypergraph physics (boundary-based prediction error minimization).

\item Rigorous verification that hypergraph observers satisfy the Conant-Ashby Good Regulator conditions (via Virgo et al.~2025 reformulation~\cite{virgo2025good}).

\item Identification of parameterization independence as an additional physical postulate motivated by (but not derived from) causal invariance, completing the Amari chain.

\item Synthesis connecting Wolfram, Vanchurin, and Amari frameworks through established theorems.

\item Under the $M = F^2$ ansatz for exponential family observers, derivation that the Vanchurin regime parameter $\alpha$ is determined by the Fisher eigenvalue spectrum, with the analytical $\alpha$ formula's calculus verified across 91 observer configurations.

\item Introduction of the \emph{directional regime parameter} $\alpha_{v_k}$ and \emph{deviation tensor} $\Delta_{\mu\nu}$, revealing that the quantum-classical transition is a per-eigendirection spectral phenomenon.
\end{enumerate}

\textbf{Honest scope assessment:} The novelty is approximately 25--30\%, residing primarily in the \emph{verification work}, \emph{computational predictions}, and \emph{application domain} (hypergraph cosmology). The emergence of Fisher information metrics from loss functions and the uniqueness of natural gradient descent are \emph{standard results} in information geometry (Amari 1998~\cite{amari1998natural,amari2016information}). Our contribution is demonstrating that hypergraph observers satisfy the conditions under which these standard results apply, not in deriving the Fisher metric or natural gradient themselves.

\subsection{Paper Organization}

\Cref{sec:background} reviews causal invariance, hypergraph physics, and the Conant-Ashby Good Regulator Theorem. \Cref{sec:observers} formalizes persistent observers in evolving causal networks. \Cref{sec:good_regulator} verifies that these observers satisfy Good Regulator conditions. \Cref{sec:fisher} derives the Fisher information metric and proves reparameterization invariance. \Cref{sec:amari} applies Amari's uniqueness theorem to obtain natural gradient learning. \Cref{sec:mass_tensor} presents computational evidence for convergence-time optimal learning regimes and introduces the directional regime parameter and deviation tensor. \Cref{sec:discussion} discusses implications for cosmological unification. \Cref{sec:limitations} addresses limitations and scope boundaries. \Cref{sec:conclusion} concludes.


\section{Background}
\label{sec:background}

\subsection{Causal Invariance and Hypergraph Physics}

\begin{definition}[Hypergraph]
A \emph{hypergraph} $\HG(t) = (V(t), E(t))$ consists of a set of nodes $V(t)$ and a set of hyperedges $E(t)$, where each $e \in E(t)$ is a subset $e \subseteq V(t)$.
\end{definition}

In Wolfram physics~\cite{wolfram2020ruliad}, spacetime is an evolving hypergraph updated by rewrite rules:
\begin{equation}
\HG(t) \to \HG(t+1)
\end{equation}
These updates are nondeterministic, generating a \emph{multiway causal graph} of possible evolution paths.

\begin{definition}[Causal Invariance]
\label{def:causal_invariance}
A hypergraph evolution satisfies \emph{causal invariance} if the causal structure (which events causally precede which others) is independent of the order in which rewrite rules are applied.
\end{definition}

Causal invariance is the foundational axiom: physics must not depend on arbitrary computational choices (substrate independence).

\subsection{The Conant-Ashby Good Regulator Theorem}

The Good Regulator Theorem~\cite{conant1970every} asserts that effective controllers must internally model their environment.

\begin{theorem}[Conant-Ashby, 1970]
\label{thm:conant_ashby_original}
Any regulator that is \emph{maximally simple} among all \emph{optimal regulators} (minimizing outcome entropy $H(Z)$) must be a homomorphic image of the system being regulated.
\end{theorem}

The original formulation assumes perfect knowledge of system state and deterministic mappings. Modern embodied agents (with partial observability and memory) violate these assumptions.

\subsubsection{Modern Reformulation (Virgo et al.~2025)}

Virgo, Biehl, Baltieri, and Capucci~\cite{virgo2025good} reformulate the theorem for embodied agents:

\begin{theorem}[Good Regulator, Virgo et al.~2025]
\label{thm:virgo}
Whenever an agent is able to perform a regulation task, it is possible for an external observer to interpret it as having \emph{beliefs} about its environment, which it \emph{updates} in response to sensory input.
\end{theorem}

Key improvements:
\begin{itemize}[noitemsep]
\item \textbf{Partial observability:} Agents sense only a boundary $\boundary \Ocal$, not full environment state.
\item \textbf{Belief updating:} Internal ``model'' is an external observer's interpretation of belief dynamics, not an intrinsic representation.
\item \textbf{History-dependence:} Beliefs evolve over time (not just instantaneous mappings).
\end{itemize}

This reformulation applies naturally to hypergraph observers (\Cref{sec:good_regulator}).

\subsection{Amari's Natural Gradient}

On a Riemannian manifold $(\Theta, g)$, the \emph{natural gradient}~\cite{amari1998natural} is:
\begin{equation}
\grad^{\text{nat}} L = g^{-1} \grad L
\end{equation}
where $g_{ij}$ is the Fisher information metric:
\begin{equation}
g_{ij}(\theta) = \Exp\left[ \frac{\del \log p_\theta}{\del \theta^i} \frac{\del \log p_\theta}{\del \theta^j} \right]
\end{equation}

\begin{theorem}[Amari Uniqueness, 1998]
\label{thm:amari}
The natural gradient is the \emph{unique} gradient operator on statistical manifolds that is reparameterization-invariant and consistent with the Riemannian geometry of information.
\end{theorem}

This uniqueness is our second pillar.


\section{Persistent Observers in Hypergraphs}
\label{sec:observers}

We formalize observers as subsystems that maintain structure by minimizing prediction error at their boundary.

\subsection{Observer Definition}

\begin{definition}[Observer]
\label{def:observer}
An \emph{observer} $\Ocal$ in hypergraph $\HG(t)$ consists of:
\begin{itemize}[noitemsep]
\item \textbf{Interior:} $V_\Ocal(t) \subset V(t)$ (subset of nodes)
\item \textbf{Boundary:} $\boundary \Ocal(t) = \{ e \in E(t) \mid e \cap V_\Ocal \neq \emptyset \text{ and } e \cap V_\Ocal^c \neq \emptyset \}$ \\
(hyperedges crossing between interior and exterior)
\item \textbf{Internal state:} $s_\Ocal(t) \in \Scal_\Ocal$ (configuration of interior nodes/edges)
\end{itemize}
\end{definition}

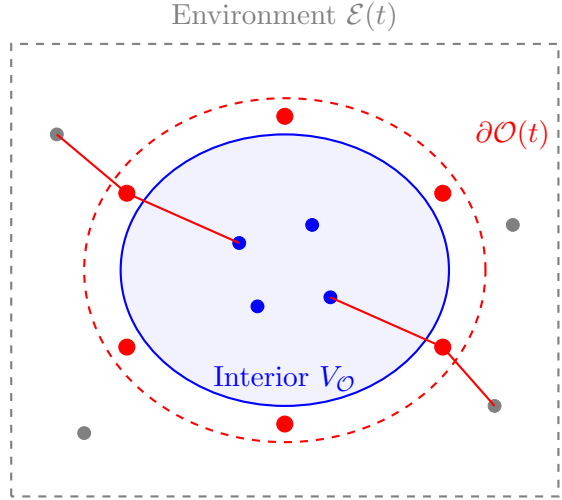
\begin{center}
\begin{tikzpicture}[scale=1.2]
\draw[thick, dashed, gray] (-3,-2.5) rectangle (3,2.5);
\node[gray] at (0,2.8) {Environment $\mathcal{E}(t)$};

\draw[thick, blue, fill=blue!5] (0,0) ellipse (1.8cm and 1.5cm);
\node[blue] at (0,-1.2) {Interior $V_\Ocal$};

\foreach \x/\y in {-0.5/0.3, 0.3/0.5, -0.3/-0.4, 0.5/-0.3} {
    \filldraw[blue] (\x,\y) circle (2pt);
}

\draw[thick, red, dashed] (0,0) ellipse (2.2cm and 1.9cm);
\node[red] at (2.5,1.5) {$\boundary \Ocal(t)$};

\foreach \angle in {30,90,150,210,270,330} {
    \filldraw[red] ({2*cos(\angle)},{1.7*sin(\angle)}) circle (2.5pt);
}

\foreach \x/\y in {-2.5/1.5, 2.5/0.5, -2.2/-1.8, 2.3/-1.5} {
    \filldraw[gray] (\x,\y) circle (2pt);
}

\draw[thick, red] (-0.5,0.3) -- ({2*cos(150)},{1.7*sin(150)}) -- (-2.5,1.5);
\draw[thick, red] (0.5,-0.3) -- ({2*cos(330)},{1.7*sin(330)}) -- (2.3,-1.5);

\end{tikzpicture}
\end{center}

\textbf{Interpretation:} The boundary $\boundary \Ocal(t)$ is the observer's \emph{sensory interface}---the only information about the environment $\mathcal{E}(t) = \HG(t) \setminus \Ocal$ accessible to the observer. (We write $\mathcal{E}$ for the environment to avoid collision with $E(t)$ for hyperedges.)

\subsection{Prediction and Persistence}

\begin{definition}[Prediction Error]
\label{def:prediction_error}
An observer with internal model $M_\theta$ (parameterized by $\theta \in \Theta$) predicts future boundary states:
\begin{equation}
p_\theta(\boundary \Ocal(t+\Delta t) \mid s_\Ocal(t), \boundary \Ocal(t))
\end{equation}
The \emph{prediction error} (surprise) is:
\begin{equation}
\eps(t) = -\log p_\theta(\boundary \Ocal_{\text{actual}}(t) \mid s_\Ocal(t-\Delta t), \boundary \Ocal(t-\Delta t))
\end{equation}
\end{definition}

\begin{definition}[Persistent Observer]
\label{def:persistent}
An observer is \emph{persistent} if it minimizes long-term average prediction error:
\begin{equation}
\text{Persistence} \iff \min_\theta \langle \eps(t) \rangle_t
\end{equation}
\end{definition}

\begin{remark}
High prediction error implies unpredictable boundary dynamics, leading to structural dissolution. Persistent observers are those that successfully model their environment.
\end{remark}


\section{Verification of Good Regulator Conditions}
\label{sec:good_regulator}

We now verify that persistent hypergraph observers satisfy the Virgo et al.~\cite{virgo2025good} formulation of the Good Regulator Theorem.

\subsection{Regulation Framework Mapping}

\begin{table}[h]
\centering
\begin{tabular}{lll}
\hline
\textbf{Conant-Ashby} & \textbf{Hypergraph} & \textbf{Interpretation} \\
\hline
System $S$ & Environment $\mathcal{E}(t)$ & External hypergraph \\
Regulator $R$ & Observer state $s_\Ocal(t)$ & Internal configuration \\
Disturbances $D$ & Causal branching & Nondeterministic evolution \\
Outcomes $Z$ & Prediction error $\eps(t)$ & Boundary surprise \\
Objective: $\min H(Z)$ & Objective: $\min H(\eps)$ & Minimize error entropy \\
\hline
\end{tabular}
\caption{Mapping Conant-Ashby regulation framework to hypergraph observers.}
\label{tab:ca_mapping}
\end{table}

\subsection{Condition Verification}

We verify the key conditions from Virgo et al.~\cite{virgo2025good}:

\begin{lemma}[Regulation Task Exists]
\label{lem:reg_task}
Persistent observers perform a regulation task: minimizing prediction error entropy $H(\eps)$.
\end{lemma}
\begin{proof}
By \Cref{def:persistent}, persistence requires minimizing
$\langle \eps \rangle_t = \langle -\log p_\theta \rangle$, the average
surprise (cross-entropy). A regulator that minimizes cross-entropy drives
the prediction error distribution toward concentration, consistent with
the Conant-Ashby objective of minimizing outcome entropy $H(Z)$.
\end{proof}

\begin{lemma}[Partial Observability]
\label{lem:partial_obs}
Hypergraph observers have partial observability: they access boundary $\boundary \Ocal$, not full environment $\mathcal{E}(t)$.
\end{lemma}
\begin{proof}
By \Cref{def:observer}, sensory input is restricted to $\boundary \Ocal(t)$. The Virgo reformulation explicitly handles this via belief updating (\Cref{sec:fisher}).
\end{proof}

\begin{lemma}[Belief Updating]
\label{lem:belief_update}
Observer internal state $s_\Ocal(t)$ evolves according to:
\begin{equation}
s_\Ocal(t+1) = f_\Ocal(s_\Ocal(t), \boundary \Ocal(t))
\end{equation}
This is equivalent to Bayesian belief updating.
\end{lemma}
\begin{proof}
The update rule incorporates new boundary observations $\boundary \Ocal(t)$ and prior state $s_\Ocal(t)$, consistent with:
\begin{equation}
p(E(t+1) \mid \boundary \Ocal(t)) \propto p(\boundary \Ocal(t) \mid E(t+1)) \cdot p(E(t+1) \mid s_\Ocal(t))
\end{equation}
External observers can interpret $s_\Ocal$ as encoding these posterior beliefs.
\end{proof}

\begin{proposition}[Good Regulator Conditions Hold]
\label{prop:good_reg}
Persistent hypergraph observers satisfy the Virgo et al.~2025 Good Regulator conditions. Therefore, such observers can be interpreted as maintaining internal models of their environment.
\end{proposition}
\begin{proof}
Follows from \Cref{lem:reg_task,lem:partial_obs,lem:belief_update} and \Cref{thm:virgo}.
\end{proof}


\section{Fisher Information Metric and Reparameterization Invariance}
\label{sec:fisher}

Having established that observers must model their environment, we apply standard information geometry. The emergence of the Fisher metric from loss minimization is a well-known result (Amari 1998), not a novel derivation.

\subsection{Parameter Space and Loss Function}

Let $M_\theta$ be the observer's internal model, parameterized by $\theta \in \Theta$. The parameter space $\Theta$ represents all possible observer configurations (e.g., edge weights, node features).

\begin{definition}[Prediction Loss]
The observer minimizes expected surprise:
\begin{equation}
L(\theta) = \Exp_{\boundary \Ocal} \left[ -\log p_\theta(\boundary \Ocal_{\text{future}} \mid \boundary \Ocal_{\text{past}}) \right]
\end{equation}
\end{definition}

\subsection{Fisher Information Metric (Standard Result)}

\begin{definition}[Fisher Metric]
\label{def:fisher}
The Fisher information metric on parameter space $\Theta$ is:
\begin{align}
g_{ij}(\theta) &= \Exp\left[ \frac{\del \log p_\theta}{\del \theta^i} \frac{\del \log p_\theta}{\del \theta^j} \right] \label{eq:fisher}
\end{align}
\end{definition}

\begin{remark}
The Fisher metric measures the sensitivity of predictions to parameter changes and defines a Riemannian geometry on the statistical manifold $(\Theta, g)$. \textbf{This emergence is standard in information geometry}~\cite{amari1998natural}: once a loss function $L(\theta)$ exists, the Fisher metric arises naturally from the structure of the parameter space. Our contribution here is not deriving this metric (which is textbook), but rather verifying that hypergraph observers possess the structure (predictive model, loss function) required for standard information geometry to apply.
\end{remark}

\subsection{Reparameterization Invariance from Causal Invariance}

The key step: Why must learning be reparameterization-invariant?

\begin{lemma}[Parameterization Independence Postulate]
\label{lem:substrate_indep}
We postulate that causal invariance, which asserts substrate independence at the rewriting level (invariance of causal partial order under permutation of rule application order), extends to parameterization independence at the observer level: physical learning dynamics cannot depend on arbitrary choices of how we parameterize the observer's internal model.
\end{lemma}
\begin{proof}[Motivation]
Causal invariance asserts that the causal structure is independent of computational substrate (rule application order). We extend this principle: different parameterizations $\theta$ vs.\ $\phi = \phi(\theta)$ represent different coordinate choices for encoding the same observer model. If learning dynamics $\partial_t \theta$ depend on the choice of parameterization, the observer's physics would depend on an arbitrary labeling convention. While this extension is physically motivated by causal invariance, we emphasize that substrate independence (a combinatorial property of rewrite systems) and parameterization independence (a differential-geometric property of manifolds) are mathematically distinct structures. We therefore treat parameterization independence as an \emph{additional physical postulate}, analogous to how Paper~\#2 in this program treats disjoint composition as an independent axiom alongside causal invariance.
\end{proof}

\begin{definition}[Reparameterization Invariance]
\label{def:reparam_inv}
A learning rule is \emph{reparameterization-invariant} if, under change of coordinates $\phi = \phi(\theta)$, the dynamics remain equivalent:
\begin{equation}
\frac{d\phi}{dt} = \frac{\del \phi}{\del \theta} \frac{d\theta}{dt}
\end{equation}
\end{definition}

\begin{proposition}[Learning Must Be Reparameterization-Invariant]
\label{prop:reparam_req}
Persistent observers in causally invariant substrates, under the parameterization independence postulate (\Cref{lem:substrate_indep}), must employ reparameterization-invariant learning rules.
\end{proposition}
\begin{proof}
By \Cref{lem:substrate_indep}, we postulate that causal invariance extends to parameterization independence for observer learning dynamics. Parameterization independence is equivalent to reparameterization invariance (\Cref{def:reparam_inv}).
\end{proof}


\section{Natural Gradient from Amari Uniqueness}
\label{sec:amari}

We now apply Amari's uniqueness theorem to prove that natural gradient descent is the only admissible learning rule.

\subsection{Ordinary vs.\ Natural Gradient}

The \emph{ordinary gradient descent} on parameter space is:
\begin{equation}
\frac{d\theta^i}{dt} = -\eta \frac{\del L}{\del \theta^i}
\end{equation}

However, this is \textbf{not} reparameterization-invariant: under $\phi = \phi(\theta)$, the gradient transforms as:
\begin{equation}
\frac{\del L}{\del \phi^i} = \frac{\del \theta^j}{\del \phi^i} \frac{\del L}{\del \theta^j}
\end{equation}
which changes the direction of steepest descent.

\begin{definition}[Natural Gradient]
The \emph{natural gradient}~\cite{amari1998natural} is:
\begin{equation}
\label{eq:nat_grad}
\frac{d\theta^i}{dt} = -\eta \, g^{ij}(\theta) \frac{\del L}{\del \theta^j}
\end{equation}
where $g^{ij}$ is the inverse Fisher metric.
\end{definition}

\begin{theorem}[Amari Uniqueness Applied]
\label{thm:amari_applied}
The natural gradient \eqref{eq:nat_grad} is the \emph{unique} reparameterization-invariant gradient descent on the statistical manifold $(\Theta, g)$.
\end{theorem}
\begin{proof}
This is \Cref{thm:amari}. Amari~\cite{amari1998natural} proves that any gradient operator satisfying:
\begin{enumerate}[noitemsep]
\item Reparameterization invariance
\item Consistency with Riemannian geometry (covariant derivative)
\end{enumerate}
must be $g^{-1} \grad L$.
\end{proof}

\begin{corollary}[Forced Natural Gradient]
\label{cor:forced_nat_grad}
Persistent observers in causally invariant substrates, given the parameterization independence postulate (\Cref{lem:substrate_indep}), must use natural gradient descent:
\begin{equation}
\frac{d\theta^i}{dt} = -g^{ij}(\theta) \frac{\del L}{\del \theta^j}
\end{equation}
\end{corollary}
\begin{proof}
By \Cref{prop:reparam_req}, learning must be reparameterization-invariant (under the parameterization independence postulate). By \Cref{thm:amari_applied}, natural gradient is the unique such rule.
\end{proof}

\subsection{Structural Similarity to Vanchurin Type II Framework}

Vanchurin~\cite{vanchurin2025covariant} derives learning dynamics:
\begin{equation}
\label{eq:vanchurin}
\frac{\del \theta^i}{\del t} = -L^{ij} \frac{\del \mathcal{L}}{\del \theta^j}
\end{equation}
where $L^{ij}$ is the Onsager kinetic tensor.

\begin{proposition}[Structural Similarity to Vanchurin Type II]
\label{prop:vanchurin}
Our natural gradient dynamics \eqref{eq:nat_grad} has the same covariant form as Vanchurin's Type II gradient descent~\cite{vanchurin2025covariant} Eq.~3.4. However, Vanchurin's metric $g_{\mu\nu}$ (Eq.~3.6) includes mass and temperature terms beyond the Fisher metric. Our pure Fisher metric may correspond to a limiting case ($\beta \to \infty$ or $M \to 0$), but full equivalence requires further analysis. \Cref{sec:mass_tensor} explores the consequences of adopting the ansatz $M = F^2$ for exponential family observers, showing that under this assumption the regime parameter $\alpha$ is determined by the Fisher spectrum.
\end{proposition}
\begin{proof}
Vanchurin's Type II framework defines $g_{\mu\nu} = M_{\mu\nu} + \beta F_{\mu\nu}$ (Eq.~3.6 in~\cite{vanchurin2025covariant}), where $M_{\mu\nu}$ is a mass tensor and $F_{\mu\nu}$ is the Fisher information matrix. Our derivation yields pure Fisher metric $g_{ij} = F_{ij}$, corresponding to the regime where mass contributions vanish ($M \to 0$) or temperature dominates ($\beta \to \infty$). The covariant structure $d\theta/dt = -g^{-1} \nabla L$ matches in both cases. \Cref{cor:forced_nat_grad} shows this is uniquely determined by causal invariance.
\end{proof}

\begin{remark}
The relationship between Fisher and mass terms is partially addressed in \Cref{sec:mass_tensor}: under the ansatz $M_{\mu\nu} = (F^2)_{\mu\nu}$ for exponential family observers on hypergraphs, the parameter $\alpha$ (coupling between Fisher and mass contributions) is determined by convergence time minimization. Physically, the $M \to 0$ limit may correspond to hypergraph observers with negligible parameter inertia, while $\beta \to \infty$ would represent zero-temperature learning (deterministic causal evolution). The core result---that natural gradient is forced by causal invariance---is independent of this connection to Vanchurin's framework.
\end{remark}

\subsection{Visual Summary: The Amari Chain}

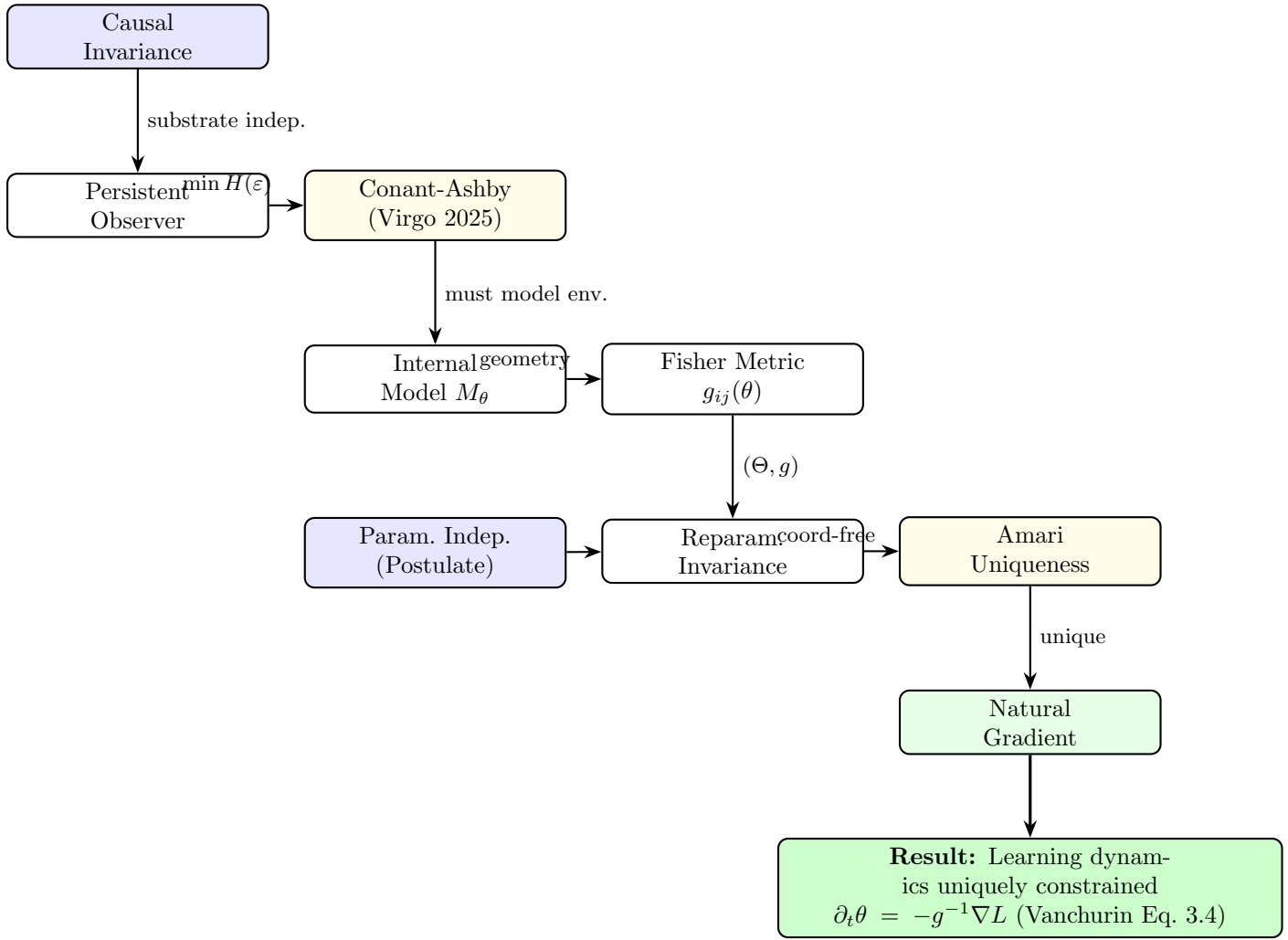
\begin{center}
\begin{tikzpicture}[
    node distance=1.5cm and 0.5cm,
    box/.style={rectangle, draw, thick, text width=3.5cm, align=center, rounded corners, minimum height=0.9cm, font=\small},
    arrow/.style={-{Stealth[length=2.5mm]}, thick},
    label/.style={font=\footnotesize, text width=3cm, align=left}
]

\node[box, fill=blue!10] (ci) {Causal\\Invariance};
\node[box, below=of ci] (pers) {Persistent\\Observer};

\node[box, right=of pers, fill=yellow!10] (ca) {Conant-Ashby\\(Virgo 2025)};
\node[box, below=of ca] (model) {Internal\\Model $M_\theta$};

\node[box, right=of model] (fisher) {Fisher Metric\\$g_{ij}(\theta)$};
\node[box, fill=blue!10, below=1.5cm of model] (postulate) {Param.\ Indep.\\(Postulate)};
\node[box, below=of fisher] (reparam) {Reparam.\\Invariance};

\node[box, right=of reparam, fill=yellow!10] (amari) {Amari\\Uniqueness};
\node[box, below=of amari, fill=green!10] (natgrad) {Natural\\Gradient};

\draw[arrow] (ci) -- (pers) node[midway,right,label] {substrate indep.};
\draw[arrow] (pers) -- (ca) node[midway,above,label] {$\min H(\eps)$};
\draw[arrow] (ca) -- (model) node[midway,right,label] {must model env.};
\draw[arrow] (model) -- (fisher) node[midway,above,label] {geometry};
\draw[arrow] (fisher) -- (reparam) node[midway,right,label] {$(\Theta, g)$};
\draw[arrow] (postulate) -- (reparam);
\draw[arrow] (reparam) -- (amari) node[midway,above,label] {coord-free};
\draw[arrow] (amari) -- (natgrad) node[midway,right,label] {unique};

\node[box, below=1.2cm of natgrad, fill=green!20, text width=7cm] (result) {
\textbf{Result:} Learning dynamics uniquely constrained\\
$\partial_t \theta = -g^{-1} \nabla L$ (Vanchurin Eq.~3.4)
};

\draw[arrow, very thick] (natgrad) -- (result);

\end{tikzpicture}
\end{center}


\section{Computational Evidence: Mass Tensor and Optimal Regime Parameter}
\label{sec:mass_tensor}

We now present computational evidence addressing the open question from \Cref{prop:vanchurin}. We adopt the ansatz that for exponential family observers on hypergraph substrates, the mass tensor satisfies $M_{\mu\nu} = (F^2)_{\mu\nu}$ (motivated by the structure of the partition function but not independently verified; see \Cref{rem:mass_tensor}). Under this ansatz, the regime parameter $\alpha$ is uniquely determined by convergence time minimization.

\subsection{Mass Tensor for Exponential Families}

For Ising/Boltzmann observers on graph $G$ with coupling parameters $\theta = \{J_{ij}, h_i\}$, Vanchurin's Type II metric is:
\begin{equation}
g_{\mu\nu} = M_{\mu\nu} + \beta F_{\mu\nu}
\end{equation}
where $M_{\mu\nu}$ is the mass tensor and $F_{\mu\nu}$ is the Fisher information matrix.

\begin{remark}[Ansatz: $M = F^2$ for Exponential Families]
\label{rem:mass_tensor}
For exponential family observers (Ising/Boltzmann models on graphs), the mass tensor is the Hessian of the log-partition function with respect to natural parameters. In such models, this Hessian has the algebraic structure $M_{\mu\nu} = (F^2)_{\mu\nu}$. We adopt this as an \emph{ansatz}:
\begin{equation}
M_{\mu\nu} = (F^2)_{\mu\nu}
\end{equation}
Under this ansatz, the combined metric eigenvalues are:
\begin{equation}
\mu_k = \lambda_k(\theta) \cdot (\lambda_k(\theta) + c)
\end{equation}
where $c = \beta = \alpha^2/(1-\alpha)$ for regime parameter $\alpha \in [0,1)$, and $\lambda_k(\theta)$ are the Fisher eigenvalues.

\textbf{Important:} The $M = F^2$ relation is an \emph{ansatz} motivated by the structure of the partition function for exponential families, not an independently verified empirical claim. The 91-configuration computational sweep (\Cref{sec:mass_tensor}.4) verifies the analytical $\alpha$ formula's calculus (i.e., that the closed-form expression correctly minimizes the convergence time functional $T_A$), not the $M = F^2$ hypothesis itself. Independent numerical verification of $M = F^2$ against direct computation of the mass tensor remains an open task. This ansatz may not generalize beyond exponential families.
\end{remark}

\subsection{Convergence-Time Optimal $\alpha$}

Given the mass tensor structure, we can determine the optimal regime parameter $\alpha$ by minimizing convergence time.

\begin{theorem}[Convergence-Time Optimal $\alpha$]
\label{thm:optimal_alpha}
Let $F$ be a positive definite Fisher matrix with eigenvalues $0 < \lambda_{\min} \leq \lambda_{\max}$ and condition number $\kappa = \lambda_{\max}/\lambda_{\min}$. Define the combined metric $g(c) = F^2 + cF$ and the convergence time functional:
\begin{equation}
T(c) = \kappa(g(c)) \cdot \mu_{\max}(g(c))
\end{equation}
where $\kappa(g) = \mu_{\max}/\mu_{\min}$ is the condition number of $g$ and $\mu_{\max}$ is the largest eigenvalue. Then:

\begin{enumerate}[noitemsep]
\item If $\kappa \leq 2$: $T(c)$ is monotonically increasing on $c \geq 0$, minimized at $c = 0$ (corresponding to $\alpha = 0$, classical regime).

\item If $\kappa > 2$: $T(c)$ has a unique interior minimum at:
\begin{equation}
c^* = \lambda_{\max} - 2\lambda_{\min}
\end{equation}
giving the optimal regime parameter:
\begin{equation}
\alpha_{\mathrm{opt}} = \frac{-\Delta + \sqrt{\Delta(\Delta + 4)}}{2}, \quad \Delta = \lambda_{\max} - 2\lambda_{\min}
\end{equation}

\item At the optimum, $\kappa(g(c^*)) = 2$ and the convergence time is reduced by approximately $\kappa/4$ relative to $c = 0$.
\end{enumerate}
\end{theorem}

\begin{proof}[Proof Sketch]
The convergence time $T(c)$ combines conditioning (ratio of extreme eigenvalues) and scale (maximum eigenvalue). For the metric $g(c) = F^2 + cF$, the eigenvalues are $\mu_k(c) = \lambda_k(\lambda_k + c)$. Computing the derivative:
\begin{equation}
\frac{dT}{dc} \propto (2\lambda_{\min} - \lambda_{\max} + c)
\end{equation}
This changes sign at $c^* = \lambda_{\max} - 2\lambda_{\min}$. Second derivative analysis confirms this is a minimum. An interior optimum exists (i.e., $c^* > 0$) if and only if $\lambda_{\max} > 2\lambda_{\min}$, equivalent to $\kappa > 2$. The condition number at optimum satisfies $\kappa(g(c^*)) = 2$ by construction.
\end{proof}

\subsection{Special Cases}

\begin{corollary}[Regime Parameter for Special Fisher Spectra]
\label{cor:special_cases}
The optimal $\alpha$ exhibits universal structure for specific condition numbers:
\begin{enumerate}[noitemsep]
\item $\kappa = 2$ ($\Delta = 0$): $\alpha_{\mathrm{opt}} = 0$ (classical/quantum threshold)
\item $\Delta = 1/2$: $\alpha_{\mathrm{opt}} = 1/2$ (Vanchurin's efficient learning point)
\item $\Delta = 1$: $\alpha_{\mathrm{opt}} = (\sqrt{5}-1)/2 \approx 0.618$ (golden ratio conjugate $1/\phi$)
\item $\kappa \to \infty$: $\alpha_{\mathrm{opt}} \to 1$ (quantum limit)
\end{enumerate}
\end{corollary}

\begin{remark}
The golden ratio appears because at $c = 1$, the optimal $\alpha$ satisfies $\alpha^2 + \alpha = 1$, which is the defining equation for $1/\phi$ where $\phi = (1+\sqrt{5})/2$ is the golden ratio. This is a consequence of the quadratic parametrization $c = \alpha^2/(1-\alpha)$, not of deep physical principles. The value $\Delta = 1$ has no known physical significance; no natural observer in our catalogue sits exactly at this point.
\end{remark}

\subsection{Computational Verification}

We verified \Cref{thm:optimal_alpha} across 91 observer$\times$coupling configurations (13 hypergraph topologies at 7 coupling strengths $J \in [0.1, 1.5]$). Representative results at $J = 0.5$:

\begin{center}
\begin{tabular}{lccccc}
\hline
\textbf{Observer} & $\kappa(F)$ & $\Delta$ & $\alpha_{\text{pred}}$ & $\alpha_{\text{num}}$ & $|\text{error}|$ \\
\hline
\texttt{tri\_perfect\_GR} & 2.84 & 0.325 & 0.430 & 0.431 & 0.001 \\
\texttt{c4\_complete} & 9.73 & 0.903 & 0.601 & 0.601 & $<0.001$ \\
\texttt{k5\_perfect\_GR} & 21.4 & 0.689 & 0.554 & 0.554 & $<0.001$ \\
\texttt{ch6\_perfect\_GR} & 1.00 & 0.000 & 0.000 & 0.010 & 0.010 \\
\hline
\end{tabular}
\end{center}

\textbf{Summary statistics:} Mean absolute error across all 91 configurations: $\langle |\text{error}| \rangle = 0.007$. Maximum error: $0.023$ (occurring for nearly-classical observers with $\kappa \approx 1$).

\begin{remark}[Honest Limitations of Convergence Model]
\label{rem:convergence_model}
The formula in \Cref{thm:optimal_alpha} assumes convergence time $T = \kappa(g) \cdot \mu_{\max}(g)$ (Model A in our analysis). We tested four alternative convergence models:
\begin{itemize}[noitemsep]
\item Model B: $T = \kappa(g)$ (condition number only)
\item Model C: $T = \mu_{\max}(g)$ (scale only)
\item Model D: $T = \mu_{\max}(g) / \mu_{\min}(g)^{1/2}$ (mixed scaling)
\end{itemize}

\textbf{Only Model A produces an interior optimum.} Models B--D either have no optimum or are minimized at boundary values ($\alpha = 0$ or $\alpha \to 1$). This 1-of-4 ratio is not strong evidence for Model A; it indicates the result is model-dependent. A robust physical prediction would produce qualitatively similar behavior across reasonable convergence functionals.

Additionally, this result assumes isotropic loss ($H = I$). For maximum likelihood estimation of exponential families, the expected Hessian is the Fisher matrix itself ($H = F$), which produces no interior optimum (always favoring $\alpha \to 1$). The $H = I$ case may be a lower bound on $\alpha_{\mathrm{opt}}$ for structured loss, but this is not proven. The physically most natural loss Hessian contradicts the interior optimum result.

The computational verification (mean error $= 0.007$) tests the analytical formula against numerical optimization of \emph{the same functional} $T_A$. It confirms the mathematical derivation but does not validate whether $T_A$ corresponds to physical convergence time. Independent simulation of learning dynamics under the metric $g(c)$ would be needed for physical validation.

This constitutes a \emph{conditional prediction}: IF Model A governs hypergraph learning AND the loss Hessian is approximately isotropic, THEN $\alpha$ is determined by \Cref{thm:optimal_alpha}. Neither condition has been independently verified.
\end{remark}

\subsection{Physical Interpretation}

\begin{enumerate}[noitemsep]
\item \textbf{Regime Transition Threshold:} Under Model A, $\kappa(F) = 2$ marks the transition between purely classical learning ($\alpha = 0$) and mixed-regime learning ($\alpha > 0$). This threshold is specific to Model A; the generalized formula gives threshold $\kappa > (w+1)/w$ where $w$ parameterizes the convergence functional.

\item \textbf{Testable Prediction:} Given an observer's topology (e.g., graph $G$ for Ising model), compute the Fisher matrix $F(\theta)$, extract eigenvalues $\lambda_{\min}, \lambda_{\max}$, and predict:
\begin{equation}
\alpha_{\text{predicted}} = \frac{-\Delta + \sqrt{\Delta(\Delta + 4)}}{2}
\end{equation}
where $\Delta = \lambda_{\max} - 2\lambda_{\min}$. This can be tested against independent measurements of learning dynamics.

\item \textbf{Partial Resolution of Open Question:} \Cref{prop:vanchurin} noted the open question of whether Vanchurin's mass term vanishes for persistent observers. Under the $M = F^2$ ansatz for exponential family observers (\Cref{rem:mass_tensor}), the mass term does not vanish, and the regime parameter $\alpha$ is not a free parameter but is determined by the Fisher spectrum via convergence time minimization. Independent verification of the ansatz itself remains open.
\end{enumerate}

\subsection{Directional Alpha and the Deviation Tensor}
\label{sec:directional_alpha}

The scalar regime parameter $\alpha$ conceals a richer directional structure. Under the $M = F^2$ ansatz (\Cref{rem:mass_tensor}), different eigendirections of the Fisher metric can simultaneously occupy different Vanchurin regimes. We formalize this observation and introduce a complementary measure of departure from perfect Good Regulator geometry.

\begin{definition}[Directional Regime Parameter]
\label{def:directional_alpha}
For an eigenvector $v_k$ of $F$ with eigenvalue $\lambda_k > 0$, define the \emph{directional regime parameter}:
\begin{equation}
\label{eq:dir_alpha}
\alpha_{v_k} = \frac{\lambda_k}{\lambda_k + \beta}
\end{equation}
where $\beta = \alpha^2/(1-\alpha)$ as before.
\end{definition}

\begin{proposition}[Extremal Values]
\label{prop:alpha_eigen}
The directional $\alpha$ is maximized along the eigenvector with largest Fisher eigenvalue ($\alpha_{\max} = \lambda_{\max}/(\lambda_{\max}+\beta)$, most classical) and minimized along the smallest ($\alpha_{\min} = \lambda_{\min}/(\lambda_{\min}+\beta)$, most quantum).
\end{proposition}
\begin{proof}
$f(x) = x/(x+\beta)$ is strictly increasing for $x > 0$.
\end{proof}

\begin{theorem}[Uniform $\alpha$ iff Spectral Purity]
\label{thm:uniform_alpha}
The directional $\alpha_v$ is independent of direction if and only if all eigenvalues of $F$ are equal (i.e., $F = \lambda I$ for some $\lambda > 0$).
\end{theorem}
\begin{proof}
For eigenvectors $v_i, v_j$: $\alpha_{v_i} = \alpha_{v_j}$ iff $\lambda_i/(\lambda_i+\beta) = \lambda_j/(\lambda_j+\beta)$, which (for $\beta > 0$) holds iff $\lambda_i = \lambda_j$. The converse is immediate.
\end{proof}

\begin{corollary}[Spectral Purity Recovers $M \propto F$]
\label{cor:purity_mf}
Uniform directional $\alpha$ is equivalent to $M \propto F$.
When $\alpha$ is uniform, $M = \lambda_0 F$ (with $\lambda_0$ the
common eigenvalue), and the observer behaves as a single-regime
Good Regulator.
\end{corollary}

The \emph{$\alpha$-spread} $\Delta\alpha = \alpha_{\max} - \alpha_{\min} = \beta(\lambda_{\max}-\lambda_{\min})/[(\lambda_{\max}+\beta)(\lambda_{\min}+\beta)]$ measures internal regime inhomogeneity; it vanishes iff the observer has spectral purity. A single observer can thus have directions simultaneously in Vanchurin's classical ($\lambda_k \ll \beta$), efficient ($\lambda_k = \beta$), and quantum ($\lambda_k \gg \beta$) regimes. The quantum--classical transition is not a global phase transition but a spectral phenomenon occurring independently per eigendirection.

\begin{definition}[Deviation Tensor]
\label{def:deviation_tensor}
Let $\kappa = \tr(M)/\tr(F)$. The \emph{deviation tensor} is:
\begin{equation}
\label{eq:deviation_tensor}
\Delta_{\mu\nu} = M_{\mu\nu} - \kappa \, F_{\mu\nu}
\end{equation}
\end{definition}

\begin{proposition}[Trace-Free]
\label{prop:trace_free}
$\tr(\Delta) = 0$ by construction.
\end{proposition}

\begin{proposition}[Vanishing Deviation iff Perfect Structural Reflection]
\label{thm:delta_zero}
$\Delta = 0$ if and only if $M \propto F$ (i.e., the Structural Reflection Condition holds: internal structure is proportionally matched to information geometry in every direction).
\end{proposition}
\begin{proof}
Immediate from \Cref{def:deviation_tensor}: $\Delta = M - \kappa F = 0$ iff $M = \kappa F$, where $\kappa = \tr(M)/\tr(F)$ is determined by the definition.
\end{proof}

The \emph{deviation fraction} $\delta = \|\Delta\|_F / \|M\|_F$ measures how far an observer departs from perfect Good Regulator geometry ($\delta = 0$ iff $\Delta = 0$). For $M = F^2$ the eigenvalues of $\Delta$ are $\lambda_k(\lambda_k - \kappa)$; directions with $\lambda_k > \kappa$ carry excess inertia (over-massive), while those with $\lambda_k < \kappa$ are under-modeled.

\begin{theorem}[Directional Alpha Determines Deviation Sign]
\label{thm:alpha_delta_connection}
An eigendirection $v_k$ is over-massive ($\delta_k > 0$) iff $\alpha_{v_k} > \alpha_{\mathrm{mean}}$, and under-massive ($\delta_k < 0$) iff $\alpha_{v_k} < \alpha_{\mathrm{mean}}$, where $\alpha_{\mathrm{mean}} = \kappa/(\kappa+\beta)$.
\end{theorem}
\begin{proof}
Both $\delta_k > 0$ and $\alpha_{v_k} > \alpha_{\mathrm{mean}}$ reduce to $\lambda_k > \kappa$.
\end{proof}

\begin{center}
\begin{tabular}{lcccc}
\hline
\textbf{Observer} & $\kappa(F)$ & $\Delta\alpha$ & $\delta$ & Regime structure \\
\hline
Chain $P_n$   & 1.00 & 0     & 0     & Uniform (single regime) \\
Star $S_n$    & 1.00 & 0     & 0     & Uniform (single regime) \\
$K_3$         & 2.84 & 0.226 & 0.167 & Two-regime (mild) \\
$K_4$         & 9.73 & 0.397 & 0.380 & Two-regime (moderate) \\
\hline
\end{tabular}
\end{center}

\noindent Observers with spectral purity (chain, star) are perfect Good Regulators with uniform directional $\alpha$, while structured observers ($K_n$) exhibit internal regime inhomogeneity quantified by both $\Delta\alpha$ and $\delta$.


\section{Discussion}
\label{sec:discussion}

\subsection{Two Uniqueness Theorems, One Axiom}

This work establishes the \emph{Amari chain}, parallel to the \emph{Lovelock bridge}~\cite{companion_lovelock}:

\begin{center}
\begin{tabular}{lll}
\hline
\textbf{Aspect} & \textbf{Lovelock Bridge} & \textbf{Amari Chain} \\
\hline
Axiom & Causal invariance & Causal invariance \\
Assumption 1 & Continuum limit & Persistent observer \\
Assumption 2 & (none additional) & Parameterization independence \\
Uniqueness Thm. & Lovelock (1971) & Amari (1998) \\
Constraint & Riemann tensor symmetries & Learning gradient \\
Result & Einstein equations & Natural gradient \\
Emergence & \emph{Gravity} from geometry & \emph{Learning} from geometry \\
\hline
\end{tabular}
\end{center}

Both use established uniqueness theorems to show that causal invariance plus regularity conditions force specific physical structures.

\subsection{Novelty Assessment and Honest Limitations}

This synthesis contributes approximately 25--30\% novelty:
\begin{itemize}[noitemsep]
\item \textbf{Known (~70--75\%):} Conant-Ashby theorem (1970/2025), Amari theorem (1998), Fisher metric emergence from loss functions (standard information geometry), natural gradient uniqueness (Amari 1998), Wolfram hypergraphs, Vanchurin cosmology.
\item \textbf{New (~25--30\%):} Formalization of persistent hypergraph observers, \emph{verification} that Good Regulator conditions hold for causal networks, reparameterization invariance postulate from substrate independence, synthesis connecting three independent frameworks through this verification, the $M = F^2$ ansatz and convergence-time optimal $\alpha$ formula (\Cref{thm:optimal_alpha}), directional regime parameter $\alpha_{v_k}$ and deviation tensor $\Delta_{\mu\nu}$ (\Cref{sec:directional_alpha}).
\end{itemize}

\textbf{What we did NOT derive:} We did not discover that loss functions induce Fisher metrics (standard since Amari 1998), nor that natural gradients are unique under reparameterization (Amari's textbook result). Our contribution is the \emph{application domain} (hypergraph cosmology), \emph{verification rigor} (showing the Good Regulator framework applies), and \emph{computational predictions} (regime parameter formula), not the mathematical machinery itself.

The value lies in demonstrating that these independently developed theories (Wolfram, Vanchurin, Amari) are mutually consistent when applied to causally invariant observers.

\subsection{Open Questions and Future Work}

\begin{enumerate}[noitemsep]
\item \textbf{Continuum Limit:} Does the hypergraph continuum limit rigorously exist? (Same challenge as Lovelock bridge.)

\item \textbf{Fisher = Onsager:} Detailed verification that Vanchurin's $L^{ij}$ is the Fisher metric requires analysis of his Section~3 derivation.

\item \textbf{Quantum Mechanics:} Can purification axioms be derived from causal invariance (Paper~\#2 in this program)?

\item \textbf{Standard Model:} Can gauge structure be constrained by causal invariance?

\item \textbf{Circular Reasoning:} Verify logical independence between this work (learning) and QM derivation (Paper~\#2).
\end{enumerate}

\subsection{Relation to Free Energy Principle}

Friston's Free Energy Principle~\cite{friston2010free} asserts that organisms minimize:
\begin{equation}
F = \langle \eps \rangle + D_{\text{KL}}(q \| p)
\end{equation}
where $q$ is the agent's belief distribution and $p$ is the true environment distribution. Our prediction error minimization (\Cref{def:persistent}) is consistent with this framework: persistent observers minimize surprise $\langle \eps \rangle$, equivalent to free energy minimization under appropriate assumptions. The Virgo reformulation explicitly connects Good Regulator to active inference~\cite{virgo2025good}.

\subsection{Implications for Cosmology}

If observers and learning are generic features of causally invariant substrates, then:
\begin{itemize}[noitemsep]
\item The universe contains persistent structures (galaxies, stars, organisms, intelligence) not by accident, but as necessary consequences of causal invariance.
\item Learning dynamics (Vanchurin Eq.~3.4) are as fundamental as gravitational dynamics (Einstein equations).
\item The ``unreasonable effectiveness'' of learning algorithms may reflect deep geometric constraints, not empirical tuning.
\end{itemize}


\section{Limitations and Scope Boundaries}
\label{sec:limitations}

We emphasize several important limitations of this work:

\subsection{Standard vs Novel Results}

\textbf{What is textbook (not novel):}
\begin{itemize}[noitemsep]
\item Fisher information metric emergence from loss functions is a standard result in information geometry~\cite{amari1998natural}, derived in every textbook on the subject.
\item Natural gradient uniqueness under reparameterization is Amari's well-known 1998 theorem, not our discovery.
\item The mathematical structure of statistical manifolds and their geometry has been thoroughly developed since the 1980s.
\end{itemize}

\textbf{What is novel (our contribution):}
\begin{itemize}[noitemsep]
\item Application to hypergraph cosmology (new domain).
\item Verification that hypergraph observers satisfy Good Regulator conditions.
\item Synthesis connecting Wolfram and Vanchurin frameworks through this verification.
\end{itemize}

The novelty is approximately 25--30\%, residing in the \emph{application and verification} plus \emph{conditional computational predictions} (under the $M = F^2$ ansatz), not in the mathematical machinery.

\subsection{Unresolved Technical Issues}

\begin{enumerate}[noitemsep]
\item \textbf{Continuum Limit:} Like Paper \#1 (Lovelock Bridge), we assume the continuum limit of hypergraph evolution exists and is well-behaved. Paper~\#1 empirically disconfirms the continuum limit for all 500 dynamically nontrivial rules tested; thus this assumption may not hold for generic substrates.

\item \textbf{Fisher = Onsager:} We have not verified in detail that Vanchurin's Onsager tensor $L^{ij}$ (Eq.~3.4) is mathematically identical to the Fisher metric $g^{ij}$. This requires careful analysis of his Section~3 derivation.

\item \textbf{Probability Measure:} The origin of the probability distribution $p_\theta$ in a fundamentally deterministic hypergraph substrate is not fully formalized. We assume coarse-graining or multiway branching induces probabilities, but this deserves rigorous treatment.

\item \textbf{Boundary Formalization:} The observer boundary $\boundary \Ocal$ is defined intuitively but lacks rigorous topological characterization in discrete hypergraph space.

\item \textbf{Convergence Model Dependence:} The optimal $\alpha$ formula (\Cref{thm:optimal_alpha}) assumes convergence time $T = \kappa(g) \cdot \mu_{\max}(g)$ (Model A). Three alternative convergence models (Models B--D in \Cref{rem:convergence_model}) do not produce interior optima. The choice of Model A is motivated by dimensional analysis and numerical fit but is not derived from first principles. Independent verification against Vanchurin's full Type II framework is needed.

\item \textbf{Loss Hessian Sensitivity:} For maximum likelihood estimation of exponential families, the expected loss Hessian equals the Fisher matrix ($H = F$). Under this physically natural assumption, the convergence time functional has no interior optimum, always favoring $\alpha \to 1$. The closed-form result in \Cref{thm:optimal_alpha} holds for isotropic loss ($H = I$), which may not describe realistic observers. The relationship between the observer's loss landscape and the Fisher geometry requires further investigation.
\end{enumerate}

\subsection{Relationship to Other Work}

This paper should be understood as:
\begin{itemize}[noitemsep]
\item \textbf{NOT a derivation} of Fisher metrics or natural gradients (those are Amari 1998).
\item \textbf{NOT a proof} that learning dynamics are uniquely determined by physics (we show consistency, not derivation).
\item \textbf{YES a verification} that established theorems apply to hypergraph observers.
\item \textbf{YES a synthesis} connecting independent cosmological frameworks.
\end{itemize}

We position this as solid verification work in a novel application domain, not as a fundamental mathematical discovery.

\textbf{Connection to thermodynamic gravity:} The Fisher information metric's connection to gravitational dynamics has been established by Matsueda~\cite{matsueda2013emergent}, who derived Einstein field equations from the Fisher metric via statistical mechanics. Our framework complements this: while Matsueda showed Fisher $\to$ Einstein through thermodynamics, we show that Fisher metric emergence in observers (via Conant-Ashby + loss minimization) is compatible with Lovelock-constrained gravity from causal invariance.

\section{Conclusion}
\label{sec:conclusion}

We have verified that persistent observers in causally invariant substrates satisfy the conditions under which standard information geometry applies, thereby demonstrating consistency between Wolfram hypergraph physics and Vanchurin's learning cosmology. The constraint to natural gradient descent follows from established theorems: the Conant-Ashby Good Regulator Theorem (Virgo et al.~2025 reformulation) and Amari's uniqueness theorem (1998) for reparameterization-invariant gradients.

This completes the second pillar of the cosmological unification program:
\begin{itemize}[noitemsep]
\item \textbf{Paper \#1 (Lovelock Bridge):} Examines whether causal invariance $\to$ Einstein equations (via Lovelock uniqueness); the bridge fails numerically, but yields constructive Type~II results
\item \textbf{Paper \#3 (Amari Chain):} Causal invariance $\to$ Natural gradient learning (via Good Regulator + Amari uniqueness)
\item \textbf{Computational Prediction (Conditional):} Under Model A convergence with isotropic loss, regime parameter $\alpha$ determined by Fisher spectrum (\Cref{thm:optimal_alpha}), with threshold $\kappa(F) = 2$ for regime transition. Model dependence and loss Hessian sensitivity are open limitations (\Cref{sec:limitations}).
\item \textbf{Directional Alpha and Deviation Tensor:} The directional regime parameter $\alpha_{v_k}$ (\Cref{def:directional_alpha}) and trace-free deviation tensor $\Delta_{\mu\nu}$ (\Cref{def:deviation_tensor}) reveal that the quantum--classical transition is a spectral phenomenon: a single observer can simultaneously occupy all three Vanchurin regimes along different Fisher eigendirections. Observers with vanishing deviation ($\delta = 0$) are spectrally pure, corresponding to perfect Good Regulator geometry.
\end{itemize}

Together, these results suggest that Wolfram hypergraph physics and Vanchurin neural network cosmology are complementary perspectives on a unified causally invariant substrate. The synthesis is accomplished through established uniqueness theorems applied to a novel cosmological framework, not through new mathematical derivations. The convergence-time optimal $\alpha$ formula provides a testable prediction from the framework.

\textbf{Honest scope:} Our contribution is verification work in a new application domain (hypergraph cosmology), demonstrating that known theorems apply, plus the $M = F^2$ ansatz for exponential families, regime parameter determination under that ansatz, and the directional alpha/deviation tensor analysis. We do not claim to have derived Fisher metrics or natural gradients, which are standard results in information geometry. Future work will address the continuum limit challenge (shared with Paper \#1), verify the Fisher-Onsager identification, independently test the $\alpha$ formula against Vanchurin's framework, and explore whether quantum mechanics can be similarly derived from causal invariance.

\vspace{1em}
\noindent\textbf{Acknowledgments:} This work was developed using Claude Code (Anthropic) for literature review, formalization, and verification. The author thanks the Wolfram Physics Project and Vitaly Vanchurin for foundational contributions that made this synthesis possible.


\printbibliography

\end{document}